\documentclass[11pt]{article}
\usepackage[utf8]{inputenc}
\usepackage[T1]{fontenc}
\usepackage[margin=1in]{geometry}
\usepackage{amsthm}
\usepackage{graphicx} 
\usepackage{natbib}
\usepackage{xcolor}
\usepackage{longtable}
\usepackage{booktabs}
\usepackage{array}
\usepackage{url}
\definecolor{darkmagenta}{HTML}{8B008B}

\usepackage{amsmath,amsfonts,bm}

\def\eqref#1{equation~\ref{#1}}

\def\1{\bm{1}}

\def\va{{\bm{a}}}
\def\vb{{\bm{b}}}

\def\ve{{\bm{e}}}

\def\vg{{\bm{g}}}

\def\vp{{\bm{p}}}
\def\vq{{\bm{q}}}

\def\vu{{\bm{u}}}
\def\vv{{\bm{v}}}

\def\vx{{\bm{x}}}
\def\vy{{\bm{y}}}

\def\mA{{\bm{A}}}

\def\mF{{\bm{F}}}
\def\mG{{\bm{G}}}
\def\mH{{\bm{H}}}
\def\mI{{\bm{I}}}

\def\mM{{\bm{M}}}
\def\mN{{\bm{N}}}

\def\mP{{\bm{P}}}

\def\mS{{\bm{S}}}

\def\mU{{\bm{U}}}

\def\mX{{\bm{X}}}

\DeclareMathAlphabet{\mathsfit}{\encodingdefault}{\sfdefault}{m}{sl}
\SetMathAlphabet{\mathsfit}{bold}{\encodingdefault}{\sfdefault}{bx}{n}

\newcommand{\R}{\mathbb{R}}

\newcommand{\paren}[1]{\left(#1\right)}
\newcommand{\norm}[1]{\left\|#1\right\|}
\newcommand{\inner}[2]{\left\langle #1, #2\right\rangle}

\DeclareMathOperator{\Tr}{Tr}

\newtheorem{theorem}{Theorem}
\newtheorem{lemma}{Lemma}

\usepackage[capitalise,noabbrev]{cleveref}

\title{Understanding the Subspace Stabilization of the Hessian and Gradient Covariance Matrix}
\author{Fangshuo Liao \\ Rice University \\ Fangshuo.Liao@rice.edu \and Anastasios Kyrillidis \\ Rice University \\ anastasios@rice.edu}
\date{September 2026}

\begin{document}

\maketitle

\begin{abstract}
The phenomenon of the top subspace stabilization of the Hessian matrix is an surprising and critical aspect in study of the second-order information of neural network training. Prior work argues that the top subspace of the Hessian stabilizes by measuring the overlap between the top subspaces of the step-wise Hessian, and explains this stabilization with diminishing parameter change in the late phase of training. In this paper, we define a new instability metric for the subspace evolution, and use it to detect subspace stabilization that is independent of the magnitude of parameter change. In the meantime, we observe that the gradient covariance matrix has a similar property of its top subspace to the Hessian. By using a between-class and within-class decomposition of the gradient covariance matrix, we identify an explicit form that gives a near-perfect approximation of the top-$(C-1)$ subspace of the Hessian and the gradient covariance matrix. In the gradient flow set-up, we show that the slow evolution of the idenfied approximation is due to the separation between the outlier and the bulk eigenvalues of the Hessian matrix, thus providing an explanation to the phenomenon of the top subspace stabilization of the Hessian matrix.
\end{abstract}

\section{Introduction}
The loss landscape of neural network training is believed to be highly nonconvex \citep{Li2018Visualizing,swirszcz2017localminimatrainingneural,Goldblum2020Truth}, posing great challenge to the optimization process that is central to modern deep learning. Yet, empirical observations do not align well with this belief: training neural networks starting from random initialization often achieves consistent loss decrease \citep{Choromanska2015Loss,keskar2017on}, especially in large models \citep{kaplan2020scalinglawsneurallanguage,hoffmann2022trainingcomputeoptimallargelanguage}. Diving deeper into the convergence behavior, \cite{schaipp2025the} demonstrates that the behavior of learning rate schedules resembles the theory developed from convex non-smooth optimization. This near-convex behavior is further justified by the follow-up work from the perspective of convex dominance \citep{bu2026convexdominancedeeplearning}. The various empirical success implies that neural network training may not be as nonconvex as believed.

From the perspective of the Hessian matrix, the near-convex behavior is justified by the low-rankness of the Hessian \citep{sagun2017eigenvalueshessiandeeplearning,gurari2018gradientdescenthappenstiny,papyan2019measurementsthreelevelhierarchicalstructure}, where the Hessian's eigenvalue spectrum is shown to have a few positive outliers with the remaining having a small magnitude. However, the eigenvalue spectrum does not explain all the behaviors in the optimization of deep learning loss. To understand the acceleration of momentum-based optimizer for training neural networks, \cite{liao2024provable} proposed the subspace partial-convexity, which requires the top subspace to remain nearly non-changing. \cite{wen2025understanding} based its analysis on the condition that the ``river" direction in the river-valley structure of the loss landscapes changes slowly to derive theoretical results of the WSD schedule consistent with the empirical phenomenon. With a even simpler setting, \cite{qiu2025scaling} derives a predictive scaling dynamic based on the quadratic model, and \cite{meterez2026defensequadraticmodel} show that training neural networks based on its local quadratic approximation preserves the performance for up to 10\% of the training period. Together, these results suggest that the stabilization of at least the top subspace of the Hessian, as demonstrated by \cite{gurari2018gradientdescenthappenstiny}, is crucial to the understanding of neural network training.

Nevertheless, despite the empirical observation in \cite{gurari2018gradientdescenthappenstiny}, it is not clear whether the behavior is consistent in different training setups, and why the top subspace stabilization of the Hessian happens. In particular, \cite{gurari2018gradientdescenthappenstiny} measures the subspace overlap between consecutive checkpoint as the cosine angle between the subspaces. Despite being a valid metric of subspace change, this quantity is sensitive to the magnitude of the parameter change between checkpoints: if optimizers does not introduce a large enough parameter update due to reasons such as small learning rate, then the top Hessian subspaces between two iterates are inherently small. This reasoning is adopted by \cite{jaiswal2025from} to theoretically justify the stabilization of the top Hessian subspace when the training loss converges. However, it remains an open question whether the Hessian top subspace stabilization happens independent of the magnitude of the parameter update, and, if so, what is the reason behind it.

To answer the question of \textit{whether the top Hessian subspace really stabilizes}, we introduce a metric that scales the subspace change with the inverse of the parameter update magnitude to more accurately reflect the Lipschitzness of the subspace movement. Using the metric, we demonstrate empirically that the standard training settings in \cite{song2025does,cohen2021gradient} indeed exhibit top Hessian subspace stabilization independent of the update magnitude. However, we also show that, by modifying the architecture of the neural network, our metric identifies the ablation of the stabilization that is not observable using the step-wise subspace overlap metric.

To answer the question of \textit{why the top Hessian subspace stabilizes}, we take a detour to first focus on the mini-batch gradient covariance matrix, which is shown to highly correlate with the Hessian matrix \citep{zhu2019the}. We observe that the gradient covariance matrix exhibit low-rank property similar to the Hessian, whose top subspace also stabilizes beyond the initial stage of training. In addition, we observe a surprising alignment between the top subspace of the Hessian and the gradient covariance matrix, despite that the two matrices are different. Focusing on the multi-class classification task, we adopt a similar approach to the with-class and between-class decomposition of the Hessian subspace in \cite{papyan2020tracesclasscrossclassstructurepervade} to show that the top subspace of both the Hessian and the gradient covariance matrix can be described by the span of the centered per-class gradient. With this explicit identifier, we adopt matrix ODE tools to show that, under gradient flow training, the subspace change is small when the largest eigenvalue of the Hessian and per-class Hessian in the \textit{bulk} subspace is small. Together, our discovery makes the following contribution:
\begin{itemize}
    \item We use the overlap and instability metric to detect settings where the top Hessian subspace stabilizes independent of the parameter update magnitude.
    \item We show that the mini-batch gradient covariance possesses similar low-rankness and top subspace stabilization property. In addition, we show an alignment between the Hessian's and gradient covariance matrix's top subspace.
    \item We give an explicit identifier for the top subspace of the Hessian and the gradient covariance matrix. Using this identifier, we give a theoretical account of the top subspace stabilization phenomenon that depend on the separation between the top and bulk eigenvalues.
\end{itemize}

\textbf{Notations.} We use small regular letters (e.g. $a$) to denote scalars, small bold-face letters (e.g. $\va$) to denote vectors, and capital bold-face letters (e.g. $\mA$) to denote matrices. For a matrix $\mM$, we use $\sigma_i(\mM)$ to denote its $i$th singular value, $\sigma_{\min}(\mM)$ to denote its minimum singular value, and $\norm{\mM}_2,\norm{\mM}_F$ to denote its operator and Frobenius norm, respectively. For a square matrix $\mM$, we use $\lambda_i(\mM)$ to denote its $i$th eigenvalue and $\lambda_{\min}(\mM)$ to denote its smallest eigenvalue. For a symmetric matrix $\mM$, we use $\text{span}(\mM)$ to denote the row/column space of $\mM$.

\section{Does Hessian Top Subspace Really Stabilize?}
We are interested in training a neural network $f\paren{\bm{\theta},\vx}$ with parameter $\bm{\theta}\in\R^p$ operating on inputs $\vx\in\R^d$. Let the training data be $\left\{\paren{\vx_i,y_i}\right\}_{i=1}^n$. Our focus is on the $C$-class classification task with $y_i\in[C]$ denote the labels, and the neural network $f\paren{\bm{\theta},\vx}$ trained over the cross-entropy loss
\[
    \mathcal{L}\paren{\bm{\theta}} := \frac{1}{n}\sum_{i=1}^n\ell\paren{f\paren{\bm{\theta},\vx_i},y_i};\quad \ell\paren{\hat{\vy},c} = - \log \hat{\vy}_{c}
\]
Here $f\paren{\bm{\theta,\vx}}$ outputs logits. For the convenience of the study, we use $\vp(\bm{\theta},\vx) = \texttt{softmax}(f\paren{\bm{\theta},\vx})$ to denote the softmax probability. 
Denote the neural network gradient and Hessian as $\vg\paren{\bm{\theta}} = \nabla_{\bm{\theta}}\mathcal{L}\paren{\bm{\theta}}$ and $\mH = \nabla^2_{\bm{\theta}}\mathcal{L}\paren{\bm{\theta}}$ respectively. Let the eigen-decomposition of $\mH\paren{\bm{\theta}}$ be $\mH = \mU\paren{\bm{\theta}}\bm{\Lambda}\paren{\bm{\theta}}\mU\paren{\bm{\theta}}^\top$ with $\bm{\Lambda}\paren{\bm{\theta}}_{i,i}\geq \bm{\Lambda}\paren{\bm{\theta}}_{i+1,i+1}$, i.e., the eigenvalues are sorted in descending order. We define the top-$K$ subspace of $\mH\paren{\bm{\theta}}$ as $\mU_K\paren{\bm{\theta}} \in \R^{p\times K}$ containing the first $K$ columns of $\mU\paren{\bm{\theta}}$ for some $K\leq p$. Given two subspaces $\mU_K$ ad $\mU_K'$, \cite{gurari2018gradientdescenthappenstiny} defines the overlap metric as 
\begin{equation}\label{eq:overlap_metric}
    \text{Overlap}\paren{\mU_K,\mU_K'} = \frac{\Tr\paren{\mU_K\mU_K^\top\mU_K'\mU_K'^\top}}{\sqrt{\Tr\paren{\mU_K\mU_K^\top}\Tr\paren{\mU_K'\mU_K'^\top}}} = \frac{1}{K}\norm{\mU_K^\top\mU_K'}_F^2
\end{equation}

\begin{figure}[t!]
    \centering
    \includegraphics[width=\linewidth]{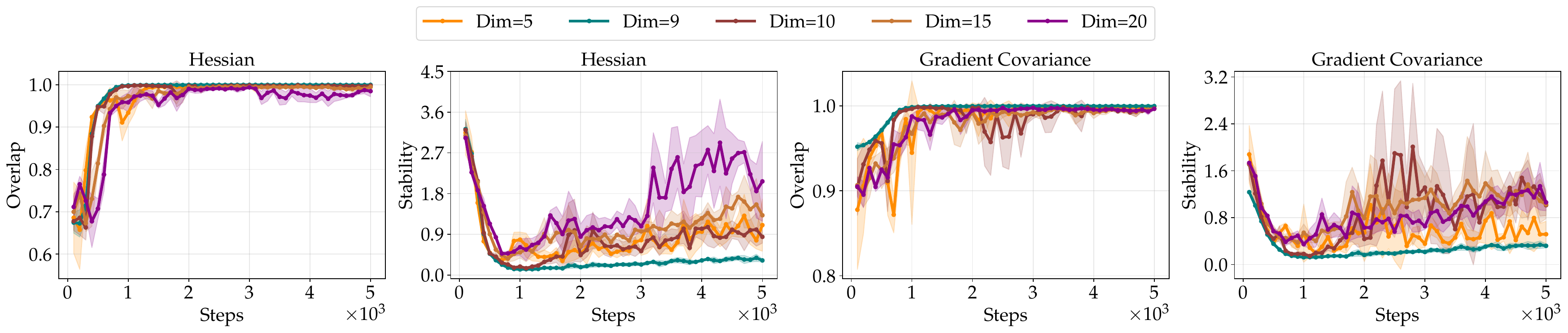}
    \includegraphics[width=\linewidth]{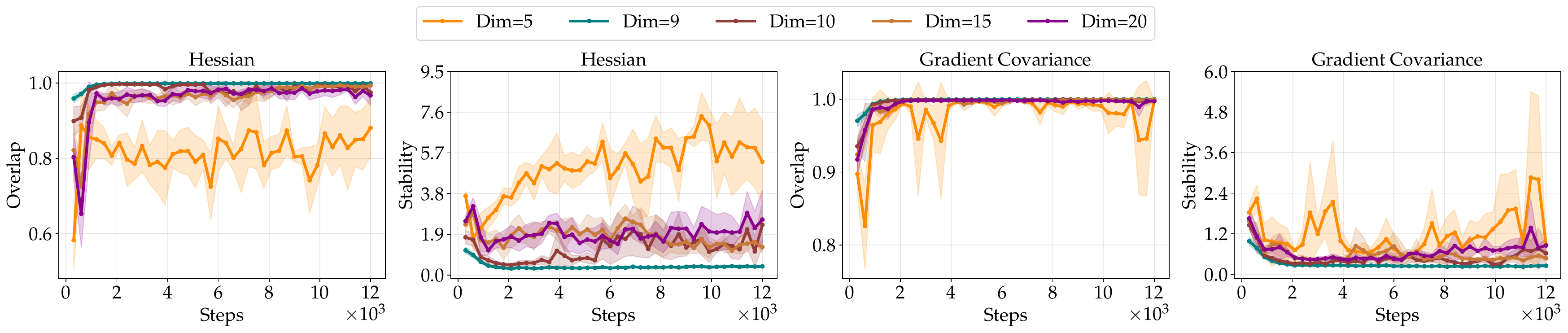}
    \includegraphics[width=\linewidth]{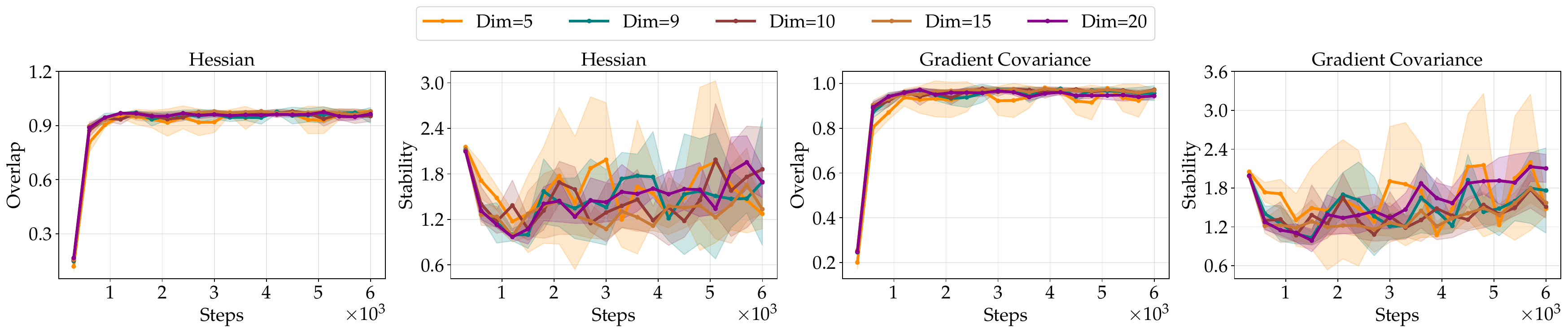}
    \caption{Top-$K$ subspace stabilization of the Hessian and the gradient covariance matrix, with $K \in \{5, 9, 10, 15, 20\}$ shown in different colors. Column 1 and 3 shows the overlap metric, and column 2 and 4 shows the stability metric. \textbf{First row:} MLP on MNIST. \textbf{Second row:} 2-Layer CNN on CIFAR-10.\textbf{Third row:} 3-Layer CNN on CIFAR-10.}
    \label{fig:hess_cov_stab}
\end{figure}
However, we note that this metric is sensitive to the choice of the learning rate. Let $\bm{\theta}$ and $\bm{\theta}'$ be two consecutive checkpoints during model training. Assuming a universal upper bound on the third-order derivative in the region of the training $\norm{\nabla^3_{\bm{\theta}}\mathcal{L}\paren{\bm{\theta}}}_2 \leq L$, we can obtain that
\begin{equation}
    \begin{aligned}
        \text{Overlap}\paren{\mU_K\paren{\bm{\theta}},\mU_K\paren{\bm{\theta}'}} & = 1 - \frac{1}{2K}\norm{\mU_K\paren{\bm{\theta}}\mU_K\paren{\bm{\theta}}^\top - \mU_K\paren{\bm{\theta}'}\mU_K\paren{\bm{\theta}'}^\top}_F^2\\
        & \geq 1 - \frac{\gamma}{2K}\norm{\mH\paren{\bm{\theta}} - \mH\paren{\bm{\theta}'}}_F^2 \geq 1 - \frac{\gamma L^2}{2K}\norm{\bm{\theta}-\bm{\theta}'}_2^2
    \end{aligned}
\end{equation}
where the first inequality is due to Davis-Kahan $\sin\Theta$ Theorem, and $\gamma$ is a quantity that depends on the eigengap of $\mH\paren{\bm{\theta}}$. Since for two consecutive checkpoints we have $\norm{\bm{\theta}-\bm{\theta}'}_2 \propto \eta$, the overlap metric in (\ref{eq:overlap_metric}) can be made arbitrarily close to $1$ by choosing a small enough learning rate. To mitigate the confounding factor of the small learning rate, we define the following instability metric that quantifies the sensitivity to perturbation between two subspaces
\begin{equation}\label{eq:stability_metric}
    \begin{aligned}
        \text{Instability}\paren{\mU_K\paren{\bm{\theta}},\mU_K\paren{\bm{\theta}'}} & := \paren{1 - \frac{1}{K}\norm{\mU_K\paren{\bm{\theta}}^\top\mU_K\paren{\bm{\theta}'}}_F^2}^{\frac{1}{2}} / \norm{\bm{\theta} - \bm{\theta}'}_2\\
        & = \frac{\norm{\mU_K\paren{\bm{\theta}}\mU_K\paren{\bm{\theta}}^\top - \mU_K\paren{\bm{\theta}'}\mU_K\paren{\bm{\theta}'}^\top}_F}{\sqrt{2K}\norm{\bm{\theta} - \bm{\theta}'}_2}
    \end{aligned}
\end{equation}
In particular, our instability metric can be interpreted as the Lipschitzness of the projection operator $\mP\paren{\bm{\theta}} = \mU_K\paren{\bm{\theta}}\mU_K\paren{\bm{\theta}}^\top$ scaled by a factor of $\frac{1}{\sqrt{2K}}$. We note that although the instability metric is not confounded by the small learning rate, it can be made arbitrarily small by a parameter change that grows in magnitude, since $\paren{1 - \frac{1}{K}\norm{\mU_K\paren{\bm{\theta}}^\top\mU_K\paren{\bm{\theta}'}}_F^2}^{\frac{1}{2}} \leq 1$ and thus $\text{Stability}\paren{\mU_K\paren{\bm{\theta}},\mU_K\paren{\bm{\theta}'}} \leq \norm{\bm{\theta} - \bm{\theta}'}_2^{-1}$. In our experiment, we rule out this possibility by plotting both $\text{Overlap}\paren{\mU_K\paren{\bm{\theta}},\mU_K\paren{\bm{\theta}'}}$ and $\text{Stability}\paren{\mU_K\paren{\bm{\theta}},\mU_K\paren{\bm{\theta}'}}$: if we observe an overlap metric $\text{Overlap}\paren{\mU_K\paren{\bm{\theta}},\mU_K\paren{\bm{\theta}'}}$ close to $1$ and a small enough $\text{Stability}\paren{\mU_K\paren{\bm{\theta}},\mU_K\paren{\bm{\theta}'}}$, then we can safely say that the top-$K$ subspace is truly stable.

With the overlap and instability metric in (\ref{eq:overlap_metric}) and (\ref{eq:stability_metric}), we test the subspace stabilization in three settings: $i)$. a two-layer MLP trained on MNIST dataset; $ii)$. a two-layer CNN trained on CIFAR-10 dataset; and $iii).$ a three-layer CNN with batch normalization trained on CIFAR-10 dataset.\footnote{The first two settings follow the standard setting in \cite{song2025does,cohen2021gradient}. Details of the experiments are provided in Appendix~\ref{app:exp_details}.} Focusing on the first two columns of Figure~\ref{fig:hess_cov_stab} that shows the stabilization property of the Hessian, we observe that in the MLP and 2-layer CNN case, pretty much every $K$ in the top-$K$ overlap (except for $K=5$ in the 2-layer CNN settings) achieves a near-perfect overlap. However, only the $K=C-1 = 9$ curve shows a consistently low stability metric. This demonstrates the insufficiently of using the overlap metric alone, which can be complemented by the instability metric we propose. Combining the overlap and instability metric, we observe that in the only the $K = C - 1 = 9$ subspace (in \textcolor{teal}{teal}) demonstrates true stability. This is possibly due to that reducing the dimensionality below $C-1$ fails to account for the rotation inside the top subspace, while increasing the dimensionality above $C-1$ introduces unstable directions. In the third column, our instability metric also discovers a setting where the top subspace fails to reach stabilization, even though the step-wise overlap between consecutive subspace remains close to 1. We will investigate this setting further in later sections. 

\begin{figure}[t!]
    \centering
    \includegraphics[width=0.8\linewidth]{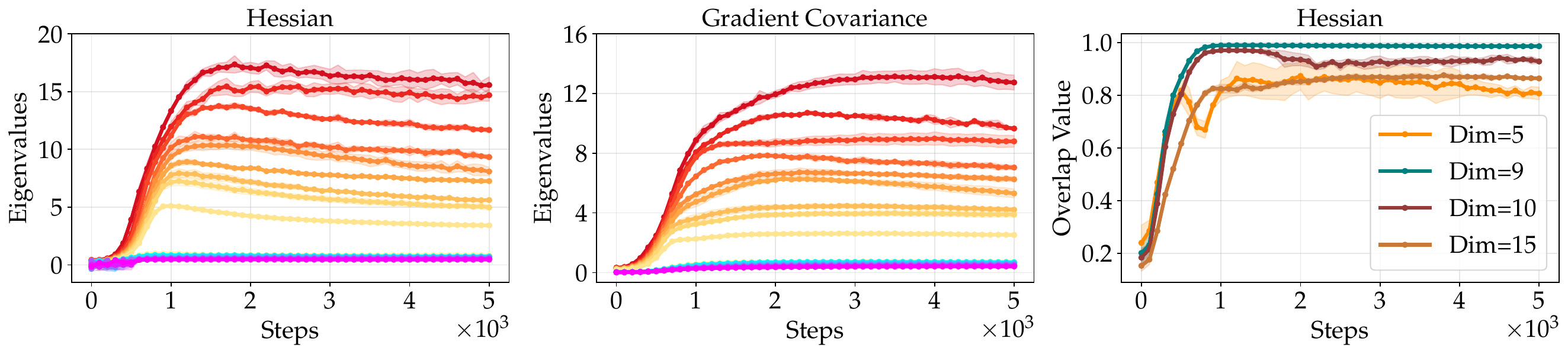}
    \includegraphics[width=0.8\linewidth]{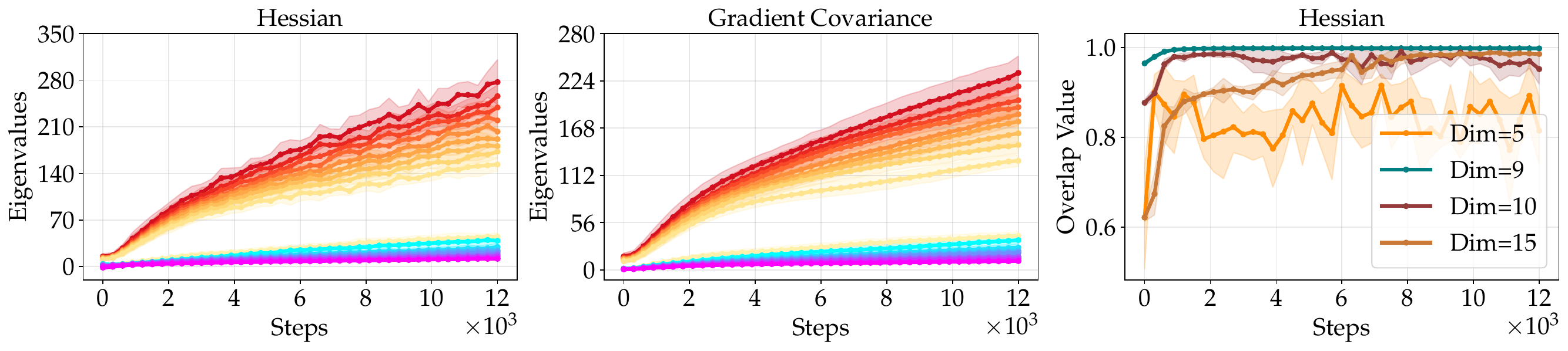}
    \includegraphics[width=0.8\linewidth]{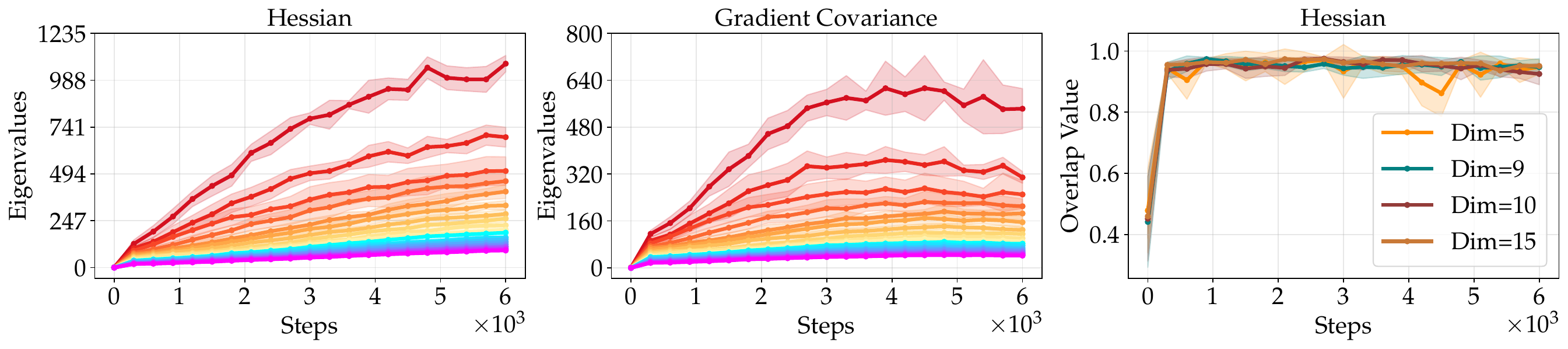}
    \caption{Structural similarity between the Hessian and the gradient covariance matrix. The three columns plots the Hessian eigenvalues, gradient covariance matrix's eigenvalues, and the overlaop between the Hessian and the gradient covariance matrix's top-$K$ subspaces. \textbf{First row:} MLP on MNIST. \textbf{Second row:} 2-Layer CNN on CIFAR-10.\textbf{Third row:} 3-Layer CNN on CIFAR-10.}
    \label{fig:hessian_cov_relationship}
\end{figure}

\section{Structure and Stabilization of Gradient Covariance Matrix}\label{sec:cov_structure_stab}

Before diving into the reasons of the stabilization of top Hessian subspace, we take a detour to explore the stabilization of a related matrix, the gradient covariance matrix. The relationship between the Hessian and the gradient covariance matrix is noted in \cite{zhu2019the}, where it is observed that for classification task, the gradient covariance matrix can be written as the in terms of the empirical Fisher information matrix (FIM), which resembles the Gauss-Newton matrix in the expansion of the Hessian. In particular, the Hessian $\mH\paren{\bm{\theta}}$ can be written in the follow decomposition
\begin{equation}
    \label{eq:hess_decomp}
    \begin{gathered}
    \mH\paren{\bm{\theta}} = \mN\paren{\bm{\theta}} + \mS\paren{\bm{\theta}};\quad \mS\paren{\bm{\theta}} := \frac{1}{n}\sum_{i=1}^n\sum_{c=1}^C\frac{\partial\ell\paren{f\paren{\bm{\theta},\vx_i},y_i}}{\partial f\paren{\bm{\theta},\vx_i}_{c}}\nabla^2_{\bm{\theta}}f\paren{\bm{\theta},\vx_i}_{c}\\
    \mN\paren{\bm{\theta}} := \frac{1}{n}\sum_{i=1}^n\sum_{c_1,c_2=1}^C\frac{\partial^2\ell\paren{f\paren{\bm{\theta},\vx_i},y_i}}{\partial f\paren{\bm{\theta},\vx_i}_{c_1}f\paren{\bm{\theta},\vx_i}_{c_2}}\nabla_{\bm{\theta}}f\paren{\bm{\theta},\vx_i}_{c_1}\nabla_{\bm{\theta}}f\paren{\bm{\theta},\vx_i}_{c_2}^\top
    \end{gathered}
\end{equation}
For classification task with cross-entropy loss, the matrix $\mN\paren{\bm{\theta}}$ is the same as the FIM
\begin{equation}\label{eq:FIM}
    \mF\paren{\bm{\theta}} := \frac{1}{n}\sum_{i=1}^n\mathbb{E}_{y\sim\vp(\bm{\theta},\vx)}\left[\nabla_{\bm{\theta}}\log \vp_{\bm{\theta}}\paren{y\mid \vx_i}\nabla_{\bm{\theta}}\log \vp_{\bm{\theta}}\paren{y\mid \vx_i}^\top\right]
\end{equation}
In particular, it has been observed in \cite{papyan2020tracesclasscrossclassstructurepervade} that the top subspace of $\mH\paren{\bm{\theta}}$ attributes to $\mN\paren{\bm{\theta}}$ (or $\mF\paren{\bm{\theta}}$. The gradient covariance matrix, on the other hand, is often associated with the \textit{empirical} Fisher information matrix $\tilde{\mF}\paren{\bm{\theta}} := \frac{1}{n}\sum_{i=1}^{n}\hat{\vg}_i\paren{\bm{\theta}}\hat{\vg}_i\paren{\bm{\theta}}^\top$\footnote{Here we are restricting the definition to the case of classification task with cross-entropy loss.}. Let $\hat{\vg}_i\paren{\bm{\theta}} = \nabla_{\bm{\theta}}\ell\paren{f\paren{\bm{\theta},\vx_i},y_i}$ denote the per-sample gradient.
Define the gradient covariance matrix
\begin{align*}
    \bm{\Sigma}(\bm{\theta}) & = \frac{1}{n}\sum_{i=1}^{n}\paren{\hat{\vg}_i\paren{\bm{\theta}} - \vg\paren{\bm{\theta}}}\paren{\hat{\vg}_i\paren{\bm{\theta}} - \vg\paren{\bm{\theta}}}^\top = \tilde{\mF}\paren{\bm{\theta}} - \vg\paren{\bm{\theta}}\vg\paren{\bm{\theta}}^\top
\end{align*}
Here $\tilde{\mF}\paren{\bm{\theta}}$ can be seen as $\mF\paren{\bm{\theta}}$ in (\ref{eq:FIM}) with deterministic ground-truth labels $y$. The overlap between $\mF\paren{\bm{\theta}}$ and $\tilde{\mF}\paren{\bm{\theta}}$ explains the overlap between $\bm{\Sigma}\paren{\bm{\theta}}$ and $\mH\paren{\bm{\theta}}$ observed in \cite{zhu2019the}. However, \cite{kunstner2019limitations} argues that the empirical FIM $\tilde{\mF}\paren{\bm{\theta}}$ does not always capture properties of $\mH\paren{\bm{\theta}}$, leaving the question of how much information of $\mH\paren{\bm{\theta}}$ resides in $\bm{\Sigma}\paren{\bm{\theta}}$ open. 

We offer a new finding of the relationship between the Hessian and the gradient covariance matrix from the perspective of their \textit{top subspace}, which has largely been ignored by previous work. From the first two columns of Figure~\ref{fig:hessian_cov_relationship}, we observe that the gradient covariance matrix has a similar eigenvalue profile to the Hessian matrix, with a clear separation of the top-$(C-1)$ outlier and the bulk in the first two rows. In the third column, we measure the overlap between the top-$K$ subspace of the Hessian and the gradient covariance matrix and the Hessian, and observes a consistent near-perfect overlap for the case of $K = C - 1 = 9$. Moreover, in column 3 and 4 in Figure~\ref{fig:hess_cov_stab}, we observe that the top-$K$ subspace of the gradient covariance matrix has a similar stability phenomenon to the top subspace of the Hessian with the most significant stability happens at $K = C - 1$. This suggest that the top-$K$ subspace of the gradient covariance matrix and the Hessian shares a similar.

\subsection{Identifying the Stable Subspace of the Gradient Covariance}
To understand what directs the evolution of the top subspace of the gradient covariance matrix, we take a similar decomposition approach as in \cite{papyan2020tracesclasscrossclassstructurepervade} for the Hessian. In particular, we decompose $\bm{\Sigma}\paren{\bm{\theta}}$ into class-specific gradient covariance matrices $\bm{\Sigma}_c\paren{\bm{\theta}}$. Recall that $C$ denotes the number of classes. Let $\mathcal{I}_c \subseteq [n]$ denote the indices of the samples belonging to class $c$ for $c\in[C]$. Let $n_c = \left|\mathcal{I}_c\right|$ denote the number of samples belonging to class $c$. Denote the class-aggregated gradient $\vg_c\paren{\bm{\theta}}$ as $\vg_c\paren{\bm{\theta}} = \frac{1}{n_c}\sum_{i\in\mathcal{I}_c}\hat{\vg}_i\paren{\bm{\theta}}$ so that $\vg\paren{\bm{\theta}} = \sum_{c=1}^C\frac{n_c}{n}\vg_c\paren{\bm{\theta}}$. We define $\bm{\Sigma}_c\paren{\bm{\theta}}$ as
\begin{equation}\label{eq:cov_c_def}
    \begin{aligned}
        \bm{\Sigma}_c\paren{\bm{\theta}} & = \frac{1}{n_c}\sum_{i\in\mathcal{I}_c}\paren{\hat{\vg}_i\paren{\bm{\theta}} - \vg\paren{\bm{\theta}}}\paren{\hat{\vg}_i\paren{\bm{\theta}} - \vg\paren{\bm{\theta}}}^\top\\
        & = \paren{\vg_c\paren{\bm{\theta}} - \vg\paren{\bm{\theta}}}\paren{\vg_c\paren{\bm{\theta}} - \vg\paren{\bm{\theta}}}^\top + \frac{1}{n_c}\sum_{i\in\mathcal{I}_c}\paren{\hat{\vg}_i\paren{\bm{\theta}} - \vg_c\paren{\bm{\theta}}}\paren{\hat{\vg}_i\paren{\bm{\theta}} - \vg_c\paren{\bm{\theta}}}^\top
    \end{aligned}
\end{equation}
Intuitively, $\bm{\Sigma}_c\paren{\bm{\theta}}$ is the per-class gradient covariance matrix. Then we have that
\begin{gather*}
    \bm{\Sigma}\paren{\bm{\theta}} = \sum_{c=1}^C\frac{n_c}{n}\bm{\Sigma}_c\paren{\bm{\theta}} = \bm{\Sigma}_{\text{top}}\paren{\bm{\theta}} + \bm{\Sigma}_{\text{bulk}}\paren{\bm{\theta}}\\
    \bm{\Sigma}_{\text{top}}\paren{\bm{\theta}} := \sum_{c=1}^C\frac{n_c}{n}\paren{\vg_c\paren{\bm{\theta}} - \vg\paren{\bm{\theta}}}\paren{\vg_c\paren{\bm{\theta}} - \vg\paren{\bm{\theta}}}^\top\\
    \bm{\Sigma}_{\text{bulk}}\paren{\bm{\theta}} := \frac{1}{n}\sum_{c=1}^C\sum_{i\in\mathcal{I}_c}\paren{\hat{\vg}_i\paren{\bm{\theta}} - \vg_c\paren{\bm{\theta}}}\paren{\hat{\vg}_i\paren{\bm{\theta}} - \vg_c\paren{\bm{\theta}}}^\top
\end{gather*}
The decomposition of the $\bm{\Sigma}\paren{\bm{\theta}}$ into $\bm{\Sigma}_{\text{top}}\paren{\bm{\theta}}$ and $\bm{\Sigma}_{\text{top}}\paren{\bm{\theta}}$ can be seen as separating the covariance into a between-class covariance $\bm{\Sigma}_{\text{top}}\paren{\bm{\theta}}$ and an average within class covariance $\bm{\Sigma}_{\text{bulk}}\paren{\bm{\theta}}$. Our hypothesis is that the between-class covariance $\bm{\Sigma}_{\text{top}}\paren{\bm{\theta}}$ explains the top-$(C-1)$ subspace of $\bm{\Sigma}\paren{\bm{\theta}}$ in the setting where the top subspace stabilizes.

To test the hypothesis, we notice that the column/row space of $\bm{\Sigma}_{\text{top}}\paren{\bm{\theta}}$ is spanned by vectors $\left\{\vg_c\paren{\bm{\theta}} - \vg\paren{\bm{\theta}}\right\}_{c=1}^{C}$. Since $\sum_{c=1}^C\frac{n_c}{n}\paren{\vg_c\paren{\bm{\theta}} - \vg\paren{\bm{\theta}}} = 0$, the same subspace must also be the span of $\left\{\vg_c\paren{\bm{\theta}} - \vg\paren{\bm{\theta}}\right\}_{c=1}^{C-1}$. Let $\mG\paren{\bm{\theta}}\in\R^{(C-1)\times p}$ be the matrix such that the $c$th row is $\vg_c\paren{\bm{\theta}} - \vg\paren{\bm{\theta}}$. In the first column of Figure~\ref{fig:G_overlaps}, we show that in the setting where the top-$(C-1)$ subspace of the Hessian and the gradient covariance matrix stabilizes (MLP and 2-layer CNN), the overlap between row space of $\mG\paren{\bm{\theta}}$ and the top-$(C-1)$ subspace of $\bm{\Sigma}\paren{\bm{\theta}}$ stays very close to $1$ throughout training. When the vectors $\left\{\vg_c\paren{\bm{\theta}} - \vg\paren{\bm{\theta}}\right\}_{c=1}^{C-1}$ are linearly independent, $\bm{\Sigma}_{\text{top}}\paren{\bm{\theta}}$ is exactly rank-$(C-1)$. Therefore, the top-$(C-1)$ subspace of $\bm{\Sigma}$ is explicitly identified by $\text{span}\paren{\vg_1\paren{\bm{\theta}} - \vg\paren{\bm{\theta}},\dots\vg_C\paren{\bm{\theta}} - \vg\paren{\bm{\theta}}}$, instead of covered by it. Thus, showing the stabilization of the top subspace of $\bm{\Sigma}\paren{\bm{\theta}}$ directly corresponding to upper bounding the change of $\text{span}\paren{\vg_1\paren{\bm{\theta}} - \vg\paren{\bm{\theta}},\dots\vg_C\paren{\bm{\theta}} - \vg\paren{\bm{\theta}}}$.

\begin{figure}[t!]
    \centering
    \includegraphics[width=0.95\linewidth]{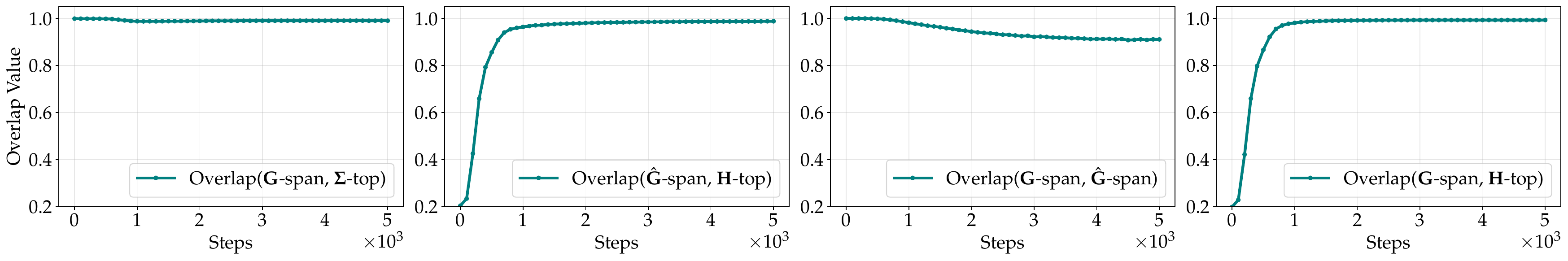}
    \includegraphics[width=0.95\linewidth]{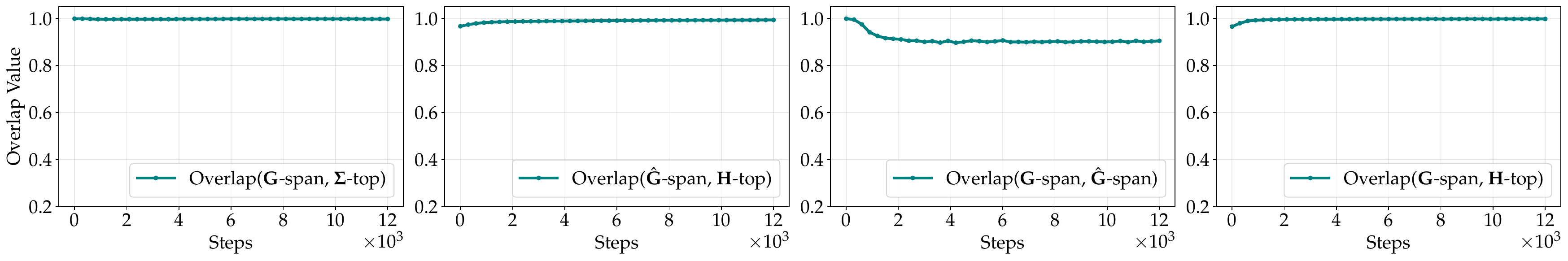}
    \includegraphics[width=0.95\linewidth]{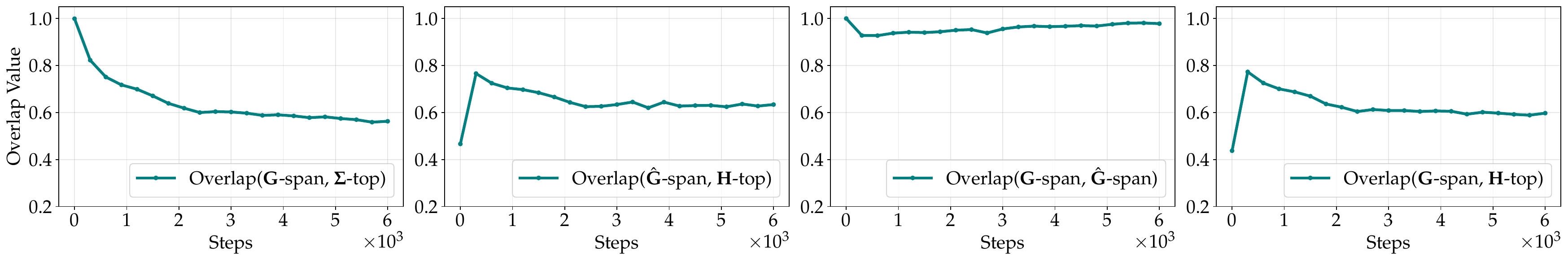}
    \caption{Overlaps between $\mG$-span, $\hat{\mG}$-span, and the top-$(C-1)$ subspace of $\bm{\Sigma}\paren{\bm{\theta}}$ and $\mH\paren{\bm{\theta}}$ throughout training. \textbf{First row:} MLP on MNIST. \textbf{Second row:} 2-Layer CNN on CIFAR-10.\textbf{Third row:} 3-Layer CNN on CIFAR-10.}
    \label{fig:G_overlaps}
\end{figure}

\subsection{Dynamic of Identified Subspace}
Next, we analyze how fast the subspace $\text{span}\paren{\vg_1\paren{\bm{\theta}} - \vg\paren{\bm{\theta}},\dots\vg_C\paren{\bm{\theta}} - \vg\paren{\bm{\theta}}}$ changes throughout training. We assume that at least $C-1$ vectors from $\left\{\vg_c\paren{\bm{\theta}} - \vg\paren{\bm{\theta}}\right\}_{c=1}^{C}$ are linearly independent, which implies that $\text{span}\paren{\vg_1\paren{\bm{\theta}} - \vg\paren{\bm{\theta}},\dots\vg_C\paren{\bm{\theta}} - \vg\paren{\bm{\theta}}}$ is a $(C-1)$-dimensional space. Indeed, this is easy to achieve in overparameterized model with $p \gg C$. To formally understand the change speed of the subspace, we study the projection matrix onto the row space of ${\vg}$, defined by
\begin{equation}\label{eq:subspace_proj}
    \mP\paren{\bm{\theta}} = {\mG\paren{\bm{\theta}}}^\top\paren{{\mG\paren{\bm{\theta}}}{\mG\paren{\bm{\theta}}}^\top}^{-1}{\mG\paren{\bm{\theta}}}\in\R^{p\times p}
\end{equation}
Consider the continuous dynamic of training with gradient flow
\begin{equation}\label{eq:gf}
    \dot{\bm{\theta}}(t) = - \nabla_{\bm{\theta}}\mathcal{L}\paren{\bm{\theta}} = - \vg\paren{\bm{\theta}(t)}
\end{equation}
We let ${\vg}(t) = {\vg}\paren{\bm{\theta}(t)}$ and correspondingly $\mP(t) = \mP(\bm{\theta}(t))$ and $\mG(t) = \mG\paren{\bm{\theta}(t)}$. Showing the stability of the subspace $\text{span}\paren{\vg_1\paren{\bm{\theta}} - \vg\paren{\bm{\theta}},\dots\vg_C\paren{\bm{\theta}} - \vg\paren{\bm{\theta}}}$ translates to bounding $\norm{\dot{\mP}(t)}_F$.
\begin{theorem}\label{thm:subspace_change}
    Let $\mP\paren{t} = \mP\paren{\bm{\theta}(t)}$ be defined in (\ref{eq:subspace_proj}). Consider the parameter dynamic induced by gradient flow on the loss $\mathcal{L}\paren{\bm{\theta}}$, as given in (\ref{eq:gf}). Then we have that
    \[
        \norm{\dot{\mP}(t)}_F^2 \leq \frac{4}{\sigma_{\min}\paren{\mG(t)}^2}\paren{(C-1)\norm{\paren{\mI - \mP(t)}\mH(t)\vg(t)}_2^2 + \sum_{c=1}^{C-1}\norm{\paren{\mI-\mP(t)}\mH_c(t)\vg(t)}_2^2}
    \]
    where $\mH(t) = \mH\paren{\bm{\theta}(t)}$ and $\mH_c(t)$ is the Hessian restricted to class-$c$ samples at $\bm{\theta}(t)$
    \[  
        \mH_c(t) := \frac{1}{n_c}\sum_{i=1}^{n_c}\nabla^2_{\bm{\theta}}\ell\paren{f\paren{\bm{\theta},\vx_{i,c}},y_c},\text{ s.t. } \mH(t) = \sum_{c=1}^C\frac{n_c}{n}\mH_c(t)
    \]
\end{theorem}
Theorem~\ref{thm:subspace_change} states that the top subspace change scales with the norm of the projected Hessian-gradient product $\paren{\mI - \mP(t)}\mH(t)\vg(t)$ and $\paren{\mI - \mP(t)}\mH_c(t)\vg(t)$. As a naive simplification, we have $\norm{\paren{\mI - \mP(t)}\mH(t)\vg(t)}_2 \propto \norm{\paren{\mI - \mP(t)}\mH(t)}_2$ and $\norm{\paren{\mI - \mP(t)}\mH_c(t)\vg(t)}_2 \propto \norm{\paren{\mI - \mP(t)}\mH_c(t)}_2$. In the theorem below, we show that $\norm{\paren{\mI - \mP(t)}\mH_c(t)}_2$ can we well controlled by the fully projected Hessian $\norm{\paren{\mI - \mP(t)}\mH(t)\paren{\mI - \mP(t)}}_2$.
\begin{theorem}\label{thm:hessian_proj} 
    Fix any $t\geq 0$ and let $\mH(t) = \mH, \mH_c(t) = \mH_c$, $\mP(t) = \mP$. Assume that $\lambda_{\min}\paren{\mH_c}\geq -\nu$ for all $c\in[C]$. Let $\varepsilon \geq \norm{\paren{\mI - \mP}\mH\paren{\mI - \mP}}_2$. Then we have that $\lambda_{\min}\paren{\mH}\geq -\nu$, and
    \[
        \norm{\paren{\mI - \mP}\mH}_2 \leq 2\sqrt{\paren{\varepsilon + \nu}\paren{\norm{\mH}_2 + \nu}};\quad \norm{\paren{\mI - \mP}\mH_c}_2 \leq \frac{2n}{n_c}\sqrt{\paren{\varepsilon + \nu}\paren{\norm{\mH}_2 + \nu}}
    \]
\end{theorem}
Theorem~\ref{thm:hessian_proj} relates the operator norm of $(\mI - \mP)\mH$ and $(\mI - \mP)\mH_c$ to $\varepsilon$, the upper bound on the Hessian's operator norm when projected onto $\mI - \mP$, which roughly corresponds to the operator norm of the Hessian in the bulk subspace. Recall that $\mP(t)$ predicts the top-$(C-1)$ subspace of $\bm{\Sigma}\paren{\bm{\theta}}$, which coincides with the top-$(C-1)$ subspace of the Hessian. Therefore, if  $\mH_c(t)$ shares a simlar subspace to $\mH(t)$, then $(\mI - \mP(t))\mH(t)$ and $(\mI - \mP(t))\mH_c(t)$ should correspond roughly to the bulk subspace of the Hessian, which is significantly smaller than the top-$(C-1)$ eigenvalues of $\mH(t)$. In the second column of Figure~\ref{fig:norm_scales}, we plot $\norm{\paren{\mI - \mP}\mH_c}_2 / \norm{\mH_c}_2$ for $c\in[C]$ in color spectrum \textcolor{yellow}{yellow} to \textcolor{red}{red}, and $\norm{\paren{\mI - \mP}\mH}_2 / \norm{\mH}_2$ in \textcolor{teal}{teal}. We observe that these ratios are consistently small after the initial period of training.

In the meantime, we point out that simply measuring $\norm{\paren{\mI - \mP}\mH}_2$ may not result in a tight bound, since $\norm{\paren{\mI - \mP}\mH\vg}_2$ and $\norm{\paren{\mI - \mP}\mH_c\vg}_2$ also involves the alignment between the gradient and the Hessian outside the top subspace. Therefore, in the third column of Figure~\ref{fig:norm_scales} we also directly measured the norms $\norm{\paren{\mI - \mP}\mH_c\vg}_2$ in the color spectrum \textcolor{yellow}{yellow} to \textcolor{red}{red} and $\norm{\paren{\mI - \mP}\mH\vg}_2$ in \textcolor{teal}{teal}, and compare directly against $\sigma_{\min}\paren{\mG(t)}$ in \textcolor{darkmagenta}{purple}. We observe that $\sigma_{\min}\paren{\mG(t)}$ is much larger in magnitude than $\norm{\paren{\mI - \mP}\mH_c\vg}_2$ and $\norm{\paren{\mI - \mP}\mH\vg}_2$.

\begin{figure}[t!]
    \centering
    \includegraphics[width=0.3\linewidth]{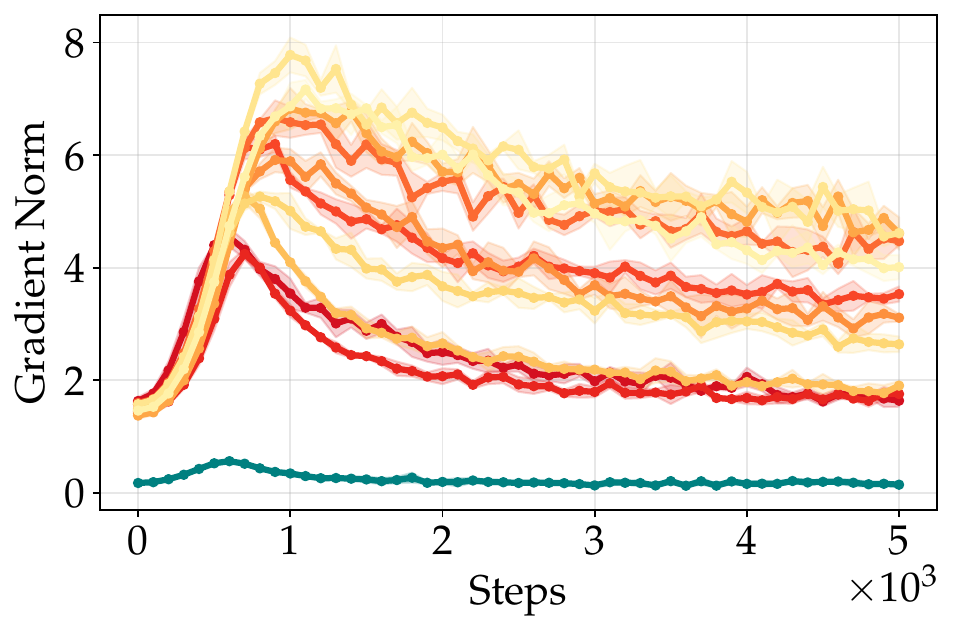}
    \includegraphics[width=0.6\linewidth]{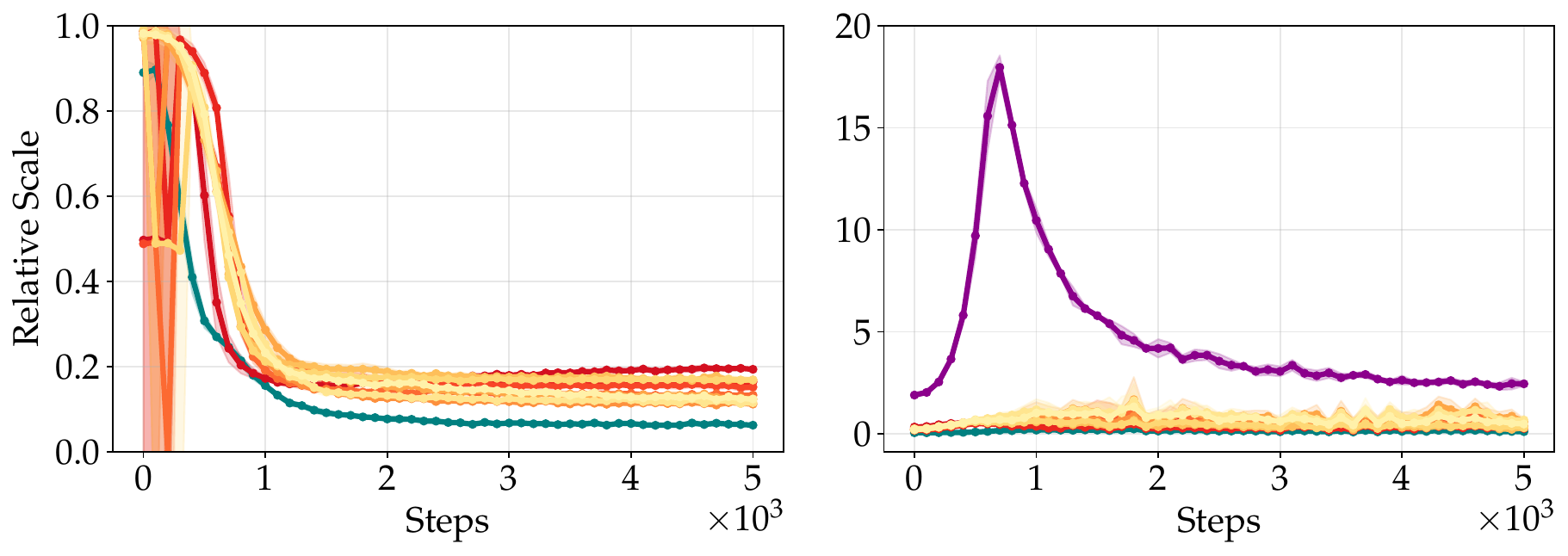}
    \includegraphics[width=0.3\linewidth]{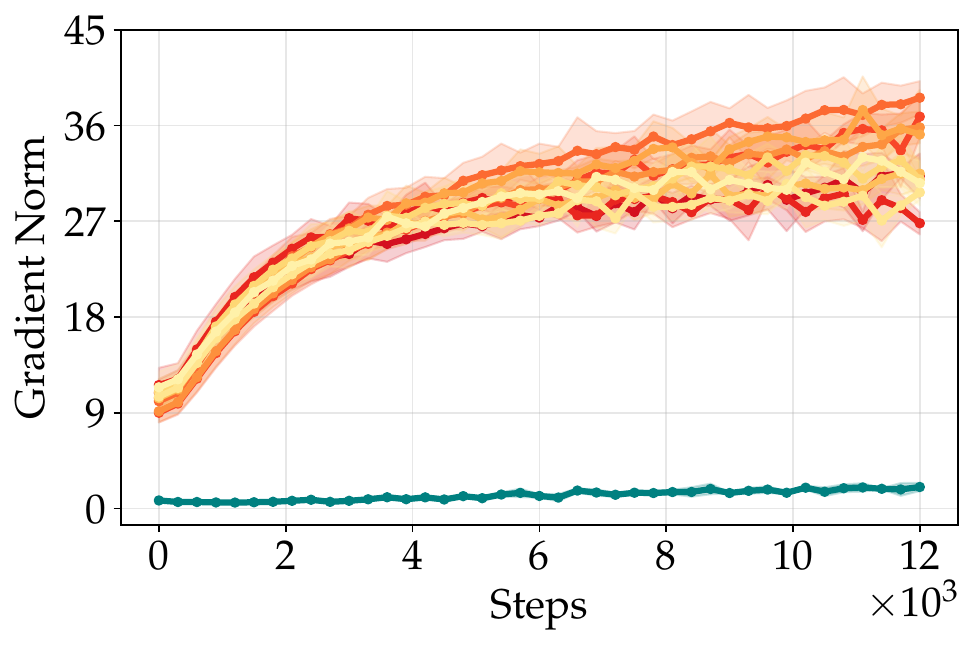}
    \includegraphics[width=0.6\linewidth]{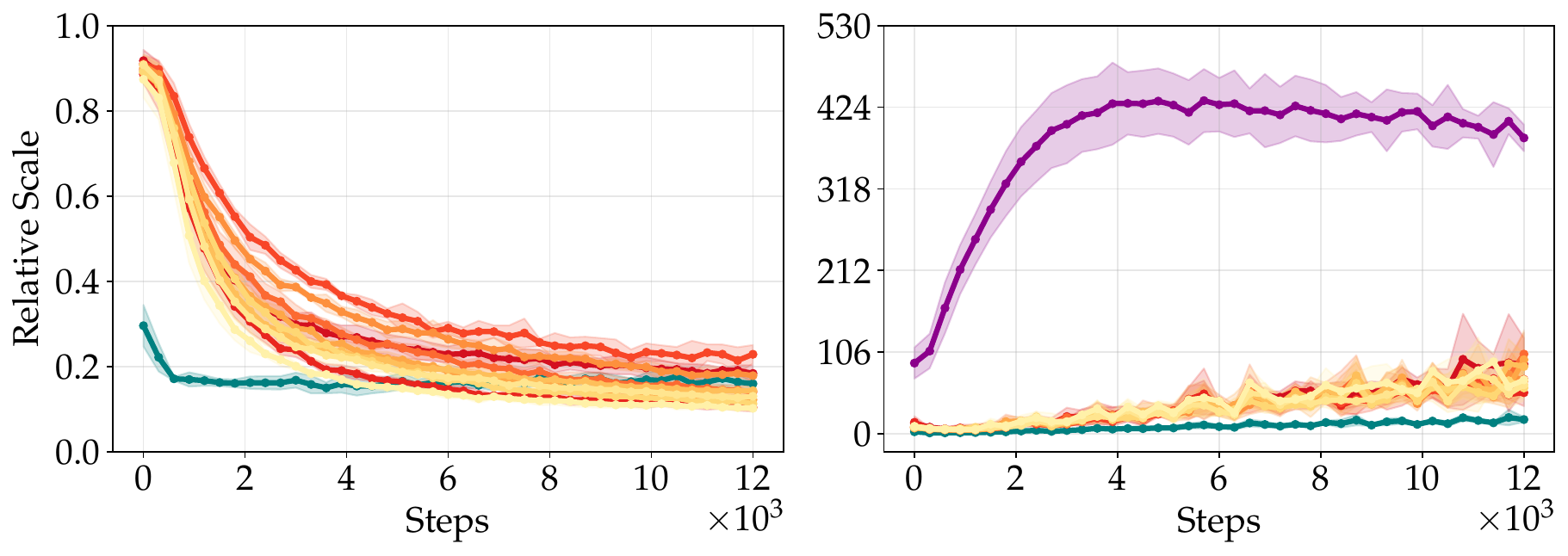}
    \caption{\textbf{Left column:} norm comparisons between $\vg_c(\bm{\theta})$s (\textcolor{yellow}{yellow} to \textcolor{red}{red} for $c\in[C]$) and $\vg\paren{\bm{\theta}}$ (\textcolor{teal}{teal}). \textbf{Mid column:} $\norm{\paren{\mI - \mP}\mH_c}_2 / \norm{\mH_c}_2$ for $c\in[C]$ (\textcolor{yellow}{yellow} to \textcolor{red}{red}), and $\norm{\paren{\mI - \mP}\mH}_2 / \norm{\mH}_2$ (\textcolor{teal}{teal}). \textbf{Right column:} comparison of $\norm{\paren{\mI - \mP}\mH_c\vg}_2$ (\textcolor{yellow}{yellow} to \textcolor{red}{red}) and $\norm{\paren{\mI - \mP}\mH\vg}_2$ (\textcolor{teal}{teal}) against $\sigma_{\min}\paren{\mG(t)}$ (\textcolor{darkmagenta}{purple}). The first and second row correspond to the setting of the first and second row in Figure~\ref{fig:hess_cov_stab}.}
    \label{fig:norm_scales}
\end{figure}

\section{Stabilization of the Hessian Top Subspace}
From the perspective of the gradient covariance matrix, Section~\ref{sec:cov_structure_stab} shows that $i)$. the top-$(C-1)$ subspace of the gradient covariance matrix has a near-perfect alignment with the top-$(C-1)$ subspace of the Hessian; $ii).$ the top-$(C-1)$ subspace of the gradient covariance matrix can be identified by $\text{span}\paren{\vg_1(\bm{\theta}) - \vg(\bm{\theta}),\dots \vg_C(\bm{\theta}) - \vg(\bm{\theta})}$; and $iii).$ the subspace $\text{span}\paren{\vg_1(\bm{\theta}) - \vg(\bm{\theta}),\dots \vg_C(\bm{\theta}) - \vg(\bm{\theta})}$ changes slowly when bulk subspace of the Hessian and the per-class Hessian are small. Therefore, it it reasonable to conclude that $\text{span}\paren{\vg_1(\bm{\theta}) - \vg(\bm{\theta}),\dots \vg_C(\bm{\theta}) - \vg(\bm{\theta})}$ serves also as an identifier for the top-$(C-1)$ subspace of the Hessian, and adopt the reason of its slow change as an explanation of the stabilization of the top-$(C-1)$ subspace of the Hessian. In this section, we provide an explanation of why $\text{span}\paren{\vg_1(\bm{\theta}) - \vg(\bm{\theta}),\dots \vg_C(\bm{\theta}) - \vg(\bm{\theta})}$ should be an identifier of the top-$(C-1)$ subspace of the Hessian based on the cross-class structure argument in \cite{papyan2020tracesclasscrossclassstructurepervade}.

\subsection{Connecting to the Subspace Identifier from \cite{papyan2020tracesclasscrossclassstructurepervade}}
\cite{papyan2020tracesclasscrossclassstructurepervade}studies the outlier subspace using a classification setup similar to our case. To ease the explanation, we use a slightly different notation from previous sections.  In particular, for the following discussion we fix $\bm{\theta}$ and omit the specification of e.g. $\vg(\bm{\theta})$ to endorse the simplified notation of $\vg$. Let $\vx_{i,c}$ denote the $i$th input sample from $c$th class, and $y_c$ be the indicator of the $c$th class. Define
\[
    \tilde{\vg}_{i,c,c'} := \nabla_{\bm{\theta}}\ell\paren{f\paren{\bm{\theta},\vx_{i,c}},y_{c'}}
\]
as the per-sample \textit{cross-class} gradient, i.e., the gradient of $\vx_{i,c}$ computed bsaed on a potentially incorrect label $y_{c'}$. Then, let $\omega_{i,c,c'} = \frac{p_{i,c,c'}}{n_cC}$ with $p_{i,c,c'} = \vp\paren{\bm{\theta},\vx_{i,c}}_{c'}$. Define auxiliary variables
\begin{equation}
    \label{eq:papyan_def}
    \begin{gathered}
        \tilde{\vg}_{c,c'} := \sum_{i\in\mathcal{I}_c}\pi_{i,c,c'}\tilde{\vg}_{i,c,c'};\quad \pi_{i,c,c'} = \frac{\omega_{i,c,c'}}{\omega_{c,c'}};\quad \omega_{c,c'} = \sum_{i\in\mathcal{I}_c}\omega_{i,c,c'}\\
        \tilde{\vg}_c := \sum_{c\neq c'}\pi_{c,c'}\vg_{c,c'};\quad \pi_{c,c'} = \frac{\omega_{c,c'}}{\omega_c};\quad \omega_c = \sum_{c\neq c'}\omega_{c,c'}
    \end{gathered}
\end{equation}
Intuitively, $\tilde{\vg}_{c,c'}$ and $\tilde{\vg}_{c}$ can be seen as a gradual weighted aggregation first over samples of each class and over cross-classes. \cite{papyan2020tracesclasscrossclassstructurepervade} shows that the top-$C$ eigenvector subspace of the Hessian attributes to the rank-$C$ matrix $\hat{\mG}:=\sum_{c=1}^C\omega_c\tilde{\vg}_c\tilde{\vg}_c^\top$. Therefore, we should also expect the row/column span of $\hat{\mG}$, namely $\text{span}(\hat{\mG})$, to cover the top-$(C-1)$ subspace of the Hessian. In cases where the top-$(C-1)$ subspaces stabilizes (first and second row of Figure~\ref{fig:G_overlaps}), it can be verified that $\hat{\mG}$ covers the top-$(C-1)$ subspace of the Hessian. However, there are two issue with this attribution:
\begin{enumerate}
    \item The directions $\vg_c'$ involves complicated definition that is not easy to interpret and analyze.
    \item The subspace $\text{span}(\hat{\mG})$ has is a $C$-dimensional subspace. Although it has been shown that $\text{span}(\hat{\mG})$ covers the subspace spanned by the top-$(C-1)$ eigenvectors of the Hessian, the dimensionality mismatch makes the characterization of the top Hessian subspace difficulty.
\end{enumerate}
\textbf{Connect to our identifier.}
Nevertheless, $\text{span}(\hat{\mG})$ provides a good starting point to be connected to our identifier $\text{span}\paren{\vg_1(\bm{\theta}) - \vg(\bm{\theta}),\dots \vg_C(\bm{\theta}) - \vg(\bm{\theta})}$. To see the connection, we start by simplifying the definition of $\vg_c$ to an equivalent set of vectors $\vg_c'$, justified by the theorem below.
\begin{theorem}
    \label{thm:simplify_g_prime}
    Let $\hat{\mG} := \sum_{c=1}^C\omega_c\tilde{\vg}_c\tilde{\vg}_c^\top$ with $\omega_c$ and $\tilde{\vg}_c$ defined in (\ref{eq:papyan_def}). Let $\vg_c' = \frac{1}{n_c}\sum_{i\in\mathcal{I}_c}\vp\paren{\bm{\theta},\vx_{i}}_c\hat{\vg}_i\paren{\bm{\theta}}$ be the aggregated gradients of samples from class-$c$, weighted by the softmax probability $\vp\paren{\bm{\theta},\vx_{i,c}}_c$ for $c\in[C]$. If $\left\{\vg_c\right\}_{c=1}^C$ is an linearly independent list, then we have that $\text{span}\paren{\vg_{\text{class}}} = \text{span}\paren{\vg_1',\dots,\vg_C'}$.
\end{theorem}
Theorem~\ref{thm:simplify_g_prime} simplified the row space of $\vg_{\text{classs}}$ into the span of vectors $\vg_c'$s, providing better interpretability of the subspace directions. Instead of relying on the notion of cross-class gradients, the definition of $\vg_c'$ is simply a weighted average of the per-sample gradients in class $c$. 

Comparing the form of $\vg_c'$ to $\vg_c$ in our setup, the difference is that our $\vg_c$ is an unweighted average of the per-sample gradients. To bridge this gap, our second observation is that the weighting probabilities $\vp\paren{\bm{\theta},\vx_{i,c}}_c$ in $\vg_c'$ can be dropped without affecting the direction of $\vg_c'$ too much. 
\begin{theorem}\label{thm:g_subspace_diff}
    Let $\mU,\mU'\in\R^{d\times C}$ denote the matrices such that the columns of $\mU$ and $\mU'$ are orthonormal basis of $\text{span}\paren{\vg_1\paren{\bm{\theta}},\dots,\vg_C\paren{\bm{\theta}}}$ and $\text{span}\paren{\vg_1'\paren{\bm{\theta}},\dots,\vg_C'\paren{\bm{\theta}}}$, respectively. Let $\mG' \in\mathbb{R}^{C\times p}$ such that the $c$th row of $\mG'$ is $\vg_c'$. Then we have that
    \[
        \sigma_{\min}\paren{\mU^\top\mU'}^2 \geq 1 - \frac{\norm{\bm{\Sigma}_{\text{bulk}}\paren{\bm{\theta}}}_2}{4\sigma_{C}\paren{\mG'}^2}
    \]
\end{theorem}
In particular, Theorem~\ref{thm:g_subspace_diff} states that the overlap between the subspace spanned by $\left\{\vg_c'\right\}_{c=1}^C$ and $\left\{\vg_c\right\}_{c=1}^C$ is off by a magnitude that scales with $\norm{\bm{\Sigma}_{\text{bulk}}\paren{\bm{\theta}}}_2$. By the form of the covariance matrix, $\norm{\bm{\Sigma}\paren{\bm{\theta}}}_2$ scales as $\norm{\vg_c - \vg}_2^2$. We observe that the full gradient $\vg$ is usually much smaller in magnitude than the per-class gradient $\vg_c$, as verified in the first column of Figure~\ref{fig:norm_scales}. In the meantime, $\vg_c$'s can be considered as well separated, as justified by the neural collapse perspective \cite{papyan20prevalence,zhu2021geometric}. This implies that $\sigma_C\paren{\mG'}^2$ roughly has the same scale as $\norm{\bm{\Sigma}\paren{\bm{\theta}}}_2$. Thus, a well-separated top-$(C-1)$ and the bulk subspace such that $\norm{\bm{\Sigma}_{\text{bulk}}\paren{\bm{\theta}}}_2\ll \norm{\bm{\Sigma}\paren{\bm{\theta}}}_2$ implies a near-perfect overlap between the span of $\left\{\vg_c'\right\}_{c=1}^C$ and $\left\{\vg_c\right\}_{c=1}^C$.

To recover our identifier $\left\{\vg_c - \vg\right\}_{c=1}^C$, it remains to explain why subtracting the full gradient from $\vg_c$ still covers the top-$(C-1)$ subspace of the Hessian. Since $\sum_{c} \frac{n_c}{n} \bm{g}_c \bm{g}_c^\top = \bm{\Sigma}_{\mathrm{top}} + \bm{g}\bm{g}^\top$, centering the class gradients removes exactly one rank-one component, along the full gradient $\bm{g}$ and of magnitude $\|\bm{g}\|_2^2$. The left column of Figure~\ref{fig:norm_scales} shows $\|\bm{g}\|_2$ to be much smaller than $\|\bm{g}_c\|_2$ throughout training, so this component is the weakest direction of the uncentered between-class moment rather than one of its top-$(C-1)$ outliers. Centering therefore discards precisely the surplus direction in the $C$-dimensional $\mathrm{span}(\hat{\bm{G}})$ while preserving the outlier subspace it covers.

\subsection{Accounting for the Reason When Top Hessian Subspace Fails to Stabilize}
In the third row of Figure~\ref{fig:hess_cov_stab}, we observed that the top subspace of the Hessian and the gradient covariance does not stabilize when we train a 3-layer CNN with batch normalization. The same setting also gives the experimental result in the third row of Figure~\ref{fig:hessian_cov_relationship} and Figure~\ref{fig:G_overlaps}. Combining the results let us infer some reasons where the top subspace fails to stabilize in this setting. 

Recall from our previous reasoning that when the top subspace stabilizes $i).$ the top subspace can be near-perfectly identified by $\text{span}(\vg_1 - \vg, \dots,\vg_C - \vg)$, and $ii).$ $\text{span}(\vg_1 - \vg, \dots,\vg_C - \vg)$ moves slowly because of the separate between the eigenvalues of the top and bulk subspace of the Hessian. From Figure~\ref{fig:hessian_cov_relationship} and Figure~\ref{fig:G_overlaps}, both conditions seems to break: in the right-most figure of the third rown in Figure~\ref{fig:hessian_cov_relationship}, we see that $\text{span}(\vg_1 - \vg, \dots,\vg_C - \vg)$ only achieves a ~0.6 overlap with the top-$(C-1)$ subspace of the Hessian, leaving a significant part of the Hessian's top-$(C-1)$
 subspace unexplained ans susceptible to chaotic forces during training. In the third row Figure~\ref{fig:hessian_cov_relationship}, we observe that the clear separation between the top and bulk subspace in both the Hessian and the gradient covariance disappears, which lead to a faster change of $\text{span}(\vg_1 - \vg, \dots,\vg_C - \vg)$, and thus a faster change of the component that $\text{span}(\vg_1 - \vg, \dots,\vg_C - \vg)$ overlaps with the top-$(C-1)$ subspace of the Hessian. While both phenomenon seems to be a predictor of when the top subspace stabilization fails, we leave the precise mechanistic understanding of this failure as a future work.

\section{Conclusion}
In this paper we study the phenomenon that the top subspace of the Hessian stabilizes shortly after the training starts. We designed a new instability metric that removes the confound of small parameter update magnitude in the overlap metric, and used it to to confirm the occurrence of the stabilization in some setting of neural network training. We also identified a setting where stabilization fails despite a near-perfect step-wise overlap. To explain the stabilization, we first investigated the relationship between the top subspace of the gradient covariance matrix and the Hessian, and constructed an explicit identifier for this subspace. Under gradient flow, we upper bound the change of this subspace by the operator norm of the Hessian restricted to the bulk, showing that the separation between outlier and bulk eigenvalues is sufficient to explain the early stabilization we observe. 

Our work leaves several open questions: $i).$ a theoretical justification of why the span of $\{\vg_c - \vg\}_{c=1}^C$ identifies the top subspace of the Hessian; $ii).$ a precise characterization of the stabilization failure in some setup; and $iii).$ an understanding of the growth of the top-$(C-1)$ eigenvalues and why they are separated from the bulk. We leave them as future directions.

\bibliography{ref}
\bibliographystyle{abbrv}

\appendix

\section{Experimental Details}\label{app:exp_details}

\subsection{Common experimental protocol}

Each experiment trains a randomly initialized classifier with plain SGD on
a fixed-size subset of a standard image classification benchmark.

\begin{itemize}
  \item \textbf{Optimizer}: vanilla SGD, no momentum (\path|momentum=0.0|), no
    weight decay, no learning-rate schedule, constant learning rate
    throughout.
  \item \textbf{Loss}: cross-entropy on the full 10-way label, both for
    training and for every Hessian/gradient statistic below.
  \item \textbf{Trials}: each setting runs 4 independent repetitions. Each trial gets a freshly initialized model and a fresh optimizer; training data (and, for CIFAR, the specific random subset of the dataset) is the same across trials within a notebook.
  \item \textbf{Results}: all results are aggregated over 4 trials, and the mean is plotted with shaded region denoting the mean$\pm$std.
\end{itemize}

\subsection{Datasets and preprocessing}
MNIST Experiment (MLP):
\begin{itemize}
    \item \textbf{Resolution}: native $28\times28$, 1 channel
    \item \textbf{Preprocessing}: raw pixel values in $[0,1]$, no standardization
    \item \textbf{Training subset}: random 10{,}000-example subset of the 60{,}000-example train split
    \item \textbf{Test set}: full standard 10{,}000-example test split (no subsetting)
\end{itemize}
CIFAR experiments (2- and 3-layer CNN)
\begin{itemize}
    \item \textbf{Resolution}: native $32\times32$, 3 channels
    \item \textbf{Preprocessing}: standardized per input dimension (channel $\times$ row $\times$ col) using the \emph{training subset's own} mean/std (computed after subsetting), applied to both train and test splits
    \item \textbf{Training subset}: random 5{,}000-example subset of the 50{,}000-example train split
    \item \textbf{Test set}: full standard 10{,}000-example test split (no subsetting)
\end{itemize}

\subsection{Model architectures}

\begin{itemize}
  \item \textbf{MLP} (MNIST): fully connected,
    biases enabled. $784 \to 256 \to 128 \to 10$, ReLU activations after
    the first two linear layers, no activation on the output layer. No
    normalization layers, no dropout.
  \item \textbf{2-Layer CNN} (CIFAR-10): two
    (\path|Conv2d(3x3, pad=1)| $\to$ ReLU $\to$ $2\times2$ max-pool)
    blocks at constant width 32 ($32\to32$ channels), \textit{no
    BatchNorm}, followed by a flatten and a single linear classifier
    head ($32\times8\times8 = 2048 \to 10$). Matches the
    ``edge-of-stability'' repository's \path|convnet| with
    \path|widths=[32, 32]|.
  \item \textbf{3-Layer CNN} (CIFAR-10): three
    (\path|Conv2d(3x3, pad=1)| $\to$ ReLU $\to$ \path|BatchNorm2d| $\to$
    $2\times2$ max-pool) blocks at constant width 32, followed by a
    flatten and a single linear classifier head
    ($32\times4\times4 = 512 \to 10$). 
\end{itemize}

\subsection{Per-experiment hyperparameters}

\begin{longtable}{@{}p{3.3cm}p{3cm}p{3cm}p{3cm}@{}}
\toprule
 & MNIST & CIFAR 2-layer  & CIFAR 3-layer \\
\midrule
\endhead
Architecture & MLP & 2-Layer CNN & 3-Layer CNN \\
Train subset size & 10{,}000 & 5{,}000 & 5{,}000 \\
Minibatch size & 100 & 50 & 50 \\
Learning rate & 0.01 & 0.001 & 0.001 \\
Total SGD steps & 5{,}000 & 12{,}000 & 6{,}000 \\
Steps / epoch & 100 & 100 & 100 \\
Epochs & 50 & 120 & 60 \\
Ckpt Freq & 100 (every epoch) & 300 (every 3 epochs) & 300 (every 3 epochs) \\
Ckpt number & 51 (initial + 50) & 41 (initial + 40) & 21 (initial + 20) \\
\bottomrule
\end{longtable}

\paragraph{Compute.} All experiments are run on Google Colab L4 and A100 GPU.

\section{Proof of Theorem~\ref{thm:subspace_change}}\label{app:thm1_proof}
\begin{proof}
    For a matrix function $\mA(t)$, its inverse's time derivative is given by
    \[
        \bm{0} = \frac{d}{dt}\paren{\mA(t)\mA^{-1}(t)} = \dot{\mA}(t)\mA^{-1}(t) + \mA(t)\dot{\mA}^{-1}(t)
    \]
    Therefore
    \[
        \dot{\mA}^{-1}(t) = - \mA^{-1}(t)\dot{\mA}(t)\mA^{-1}(t)
    \]
    Notice that $\frac{d}{dt}\paren{\mG(t)\mG(t)^\top} = \dot{\mG}(t)\mG(t)^\top +\mG(t)\dot{\mG}(t)^\top$.
    Using this formula, we have that
    \begin{align*}
        \dot{\mP}(t) & = \dot{\mG}(t)^\top\paren{\mG(t)\mG(t)^\top}^{-1}\mG(t) + \mG(t)^\top\paren{\mG(t)\mG(t)^\top}^{-1}\dot{\mG}(t)\\
        & \quad\quad\quad - \mG(t)^\top\paren{\mG(t)\mG(t)^\top}^{-1}\paren{\dot{\mG}(t)\mG(t)^\top +\mG(t)\dot{\mG}(t)^\top}\paren{\mG(t)\mG(t)^\top}^{-1}\mG(t)\\
        & = \dot{\mG}(t)^\top\paren{\mG(t)\mG(t)^\top}^{-1}\mG(t) + \mG(t)^\top\paren{\mG(t)\mG(t)^\top}^{-1}\dot{\mG}(t)\\
        & \quad\quad\quad - \mG(t)^\top\paren{\mG(t)\mG(t)^\top}^{-1}\dot{\mG}(t)\mP(t) - \mP(t)\dot{\mG}(t)^\top\paren{\mG(t)\mG(t)^\top}^{-1}\mG(t)\\
        & = \mG(t)^\top\paren{\mG(t)\mG(t)^\top}^{-1}\dot{\mG}(t)\paren{\mI - \mP(t)}\\
        & \quad\quad\quad + \paren{\mI - \mP(t)}\dot{\mG}(t)^\top\paren{\mG(t)\mG(t)^\top}^{-1}\mG(t)\\
        & = \mX(t) + \mX(t)^\top
    \end{align*}
    where in the last line we defined
    \[
        \mX(t) = \paren{\mI - \mP(t)}\dot{\mG}(t)^\top\paren{\mG(t)\mG(t)^\top}^{-1}\mG(t)
    \]
    This gives that
    \begin{align*}
        \norm{\dot{\mP}(t)}_F & \leq 2\norm{\mX(t)}_F\\
        & \leq 2\norm{\paren{\mI - \mP(t)}\dot{\mG}(t)}_F\norm{\paren{\mG(t)\mG(t)^\top}^{-1}\mG(t)}_2\\
        & \leq \frac{2}{\sigma_{\min}\paren{\mG(t)}}\norm{\paren{\mI - \mP(t)}\dot{\mG}(t)^\top}_F
    \end{align*}
    To understand $\dot{\vg}(t)$, we plug in the gradient flow dynamic to obtain that
    \begin{align*}
        \frac{d}{dt}\vg_c(t) = \mH_c(t)\dot{\bm{\theta}}(t) = -\mH_c(t)\vg(t);\;\frac{d}{dt}\hat{\vg}(t) = \mH(t)\dot{\bm{\theta}}(t) = -\mH(t)\vg(t)
    \end{align*}
    Therefore,
    \begin{align*}
        \norm{\paren{\mI - \mP(t)}\dot{\vg}(t)^\top}_F^2 & = \sum_{c=1}^{C-1}\norm{\paren{\mI - \mP(t)}\paren{\mH(t) - \mH_c(t)}\vg(t)}_2^2\\
        & \leq (C-1)\norm{\paren{\mI - \mP(t)}\mH(t)\hat{\vg}(t)}_2^2 + \sum_{c=1}^{C-1}\norm{\paren{\mI-\mP(t)}\mH_c(t)\hat{\vg}(t)}_2^2
    \end{align*}
    Pugging into the bound of $\norm{\dot{\mP}(t)}_F$ gives
    \[
        \norm{\dot{\mP}(t)}_F^2 \leq \frac{4}{\sigma_{\min}\paren{\vg(t)}^2}\paren{(C-1)\norm{\paren{\mI - \mP(t)}\mH(t)\hat{\vg}(t)}_2^2 + \sum_{c=1}^{C-1}\norm{\paren{\mI-\mP(t)}\mH_c(t)\hat{\vg}(t)}_2^2}
    \]
\end{proof}

\section{Proof of Theorem~\ref{thm:hessian_proj}}\label{app:thm2_proof}
\begin{proof}
    By definition of the minimum eigenvalue, we have that
    \begin{align*}
        \lambda_{\min}\paren{\mH} & = \min_{\vv:\norm{\vv} = 1}\vv^\top\mH\vv\\
        & = \min_{\vv:\norm{\vv}}\sum_{c=1}^C\frac{n_c}{n}\vv^\top\mH_c\vv\\
        & \geq \sum_{c=1}^C\frac{n_c}{n}\min_{\vv:\norm{\vv}}\vv^\top\mH_c\vv\\
        & = -\sum_{c=1}^C\frac{n_c}{n}\cdot \nu\\
        & = -\nu
    \end{align*}
    Therefore, we have that
    \begin{align*}
        \norm{\paren{\mI - \mP}\mH}_2 & \leq  \norm{\paren{\mI - \mP}\mH\paren{\mI - \mP}}_2 + \norm{\paren{\mI -\mP}\mH\mP}\\
        & \leq \norm{\paren{\mI - \mP}\paren{\mH + \nu\mI}\mP}_2 + \nu + \varepsilon
    \end{align*}
    Notice that $\mH + \nu\mI$ is positive semi-definite. Therefore, we can envoke Lemma~\ref{lem:matrix_cauchy} to obtain that
    \begin{align*}
        \norm{\paren{\mI - \mP}\paren{\mH + \nu\mI}\mP}_2 & \leq \norm{\paren{\mI - \mP}\paren{\mH + \nu\mI}\paren{\mI - \mP}}_2^{\frac{1}{2}}\norm{\mP\paren{\mH + \nu\mI}\mP}_2^{\frac{1}{2}}\\
        & \leq \sqrt{\paren{\varepsilon + \nu}\paren{\norm{\mH}_2 + \nu}}
    \end{align*}
    This gives that
    \[
        \norm{\paren{\mI - \mP}\mH}_2 \leq \sqrt{\paren{\varepsilon + \nu}\paren{\norm{\mH}_2 + \nu}} + \varepsilon + \nu \leq 2\sqrt{\paren{\varepsilon + \nu}\paren{\norm{\mH}_2 + \nu}}
    \]
    since $\varepsilon \leq \norm{\mH}_2$. Now, we focus on $\mH_c$. Using a similar technique, we obtain that
    \begin{align*}
        \norm{\paren{\mI - \mP}\mH_c}_2 & = \norm{\paren{\mI - \mP}\paren{\mH_c + \nu\mI}} + \nu \\
        & \leq \norm{\paren{\mI - \mP}\paren{\mH_c + \nu\mI}\paren{\mI - \mP}}_2 + \norm{\paren{\mI - \mP}\paren{\mH_c + \nu\mI}\mP}_2 + \nu
    \end{align*}
    For the first term, we notice that
    \[
        \sum_{c'=1}^C\frac{n_c'}{n}\paren{\mH_c' + \nu\mI} = \mH + \nu\mI\;\Rightarrow\; \paren{\mH+\nu \mI} - \frac{n_c}{n}\paren{\mH_c + \nu\mI} = \sum_{c'\neq c}\frac{n_c'}{n}\paren{\mH_c' + \nu\mI} \succeq 0
    \]
    since each $\mH_c' + \nu\mI$ must be positive semi-definite. Therefore, $\norm{\mH_c + \nu\mI}_2 \leq \frac{n}{n_c}\norm{\mH + \nu\mI}_2$, and
    \[
        \frac{n_c}{n}\norm{\paren{\mI - \mP}\paren{\mH_c + \nu\mI}\paren{\mI - \mP}}_2 \leq \norm{\paren{\mI - \mP}\paren{\mH + \nu\mI}\paren{\mI - \mP}}_2 \leq \nu+\varepsilon
    \]
    Thus, we have that
    \[
        \norm{\paren{\mI - \mP}\paren{\mH_c + \nu\mI}\paren{\mI - \mP}}_2\leq \frac{n}{n_c}\paren{\nu + \varepsilon} 
    \]
    Moreover, for $\norm{\paren{\mI - \mP}\paren{\mH_c + \nu\mI}\mP}_2$, by Lemma~\ref{lem:matrix_cauchy}
    \begin{align*}
        \norm{\paren{\mI - \mP}\paren{\mH_c + \nu\mI}\mP}_2 & \leq \norm{\paren{\mI - \mP}\paren{\mH_c + \nu\mI}\paren{\mI - \mP}}_2^{\frac{1}{2}}\norm{\mP\paren{\mH_c + \nu\mI}\mP}_2^{\frac{1}{2}}\\
        & \leq \frac{n}{n_c}\sqrt{(\nu+\varepsilon)\paren{\norm{\mH}_2 + \nu}}
    \end{align*}
    Combining, we have that
    \[
        \norm{\paren{\mI - \mP}\mH_c}_2 \leq \frac{n}{n_c}\paren{\nu + \varepsilon}  + \frac{n}{n_c}\sqrt{(\nu+\varepsilon)\paren{\norm{\mH}_2 + \nu}} \leq \frac{2n}{n_c}\sqrt{(\nu+\varepsilon)\paren{\norm{\mH}_2 + \nu}}
    \]
\end{proof}

\section{Proof of Theorem~\ref{thm:simplify_g_prime}}\label{app:thm3_proof}
\begin{proof}
We write
\begin{align*}
    \tilde{\vg}_c = \sum_{c\neq c'}\sum_{i\in\mathcal{I}_c}\pi_{c,c'}\pi_{i,c,c'}\tilde{\vg}_{i,c,c'} = \frac{1}{\omega_c}\sum_{c\neq c'}\sum_{i\in\mathcal{I}_c}\omega_{i,c,c'}\tilde{\vg}_{i,c,c'} = \frac{1}{\omega_c}\sum_{i\in\mathcal{I}_c}\sum_{c\neq c'}\omega_{i,c,c'}\tilde{\vg}_{i,c,c'}
\end{align*}
Define the Jacobian $\mathcal{J}_{i,c} = \frac{\partial f(\bm{\theta},\vx_{i,c})}{\bm{\theta}} \in \R^{C\times p}$. Then we have that $\tilde{\vg}_{i,c,c'}= \mathcal{J}_{i,c}^\top \paren{\vp(\bm{\theta},\vx_{i,c}) - \ve_{c'}}$. Therefore, 
\begin{align*}
    \sum_{c'=1}^{C}\omega_{i,c,c'}\vg_{i,c,c'} & = \frac{1}{n_cC}\mathcal{J}_{i,c}\sum_{c'=1}^Cp_{i,c,c'}\paren{\vp(\bm{\theta},\vx_{i,c}) - \ve_{c'}}\\
    & = \frac{1}{n_cC}\mathcal{J}_{i,c}\paren{\sum_{c'=1}^Cp_{i,c,c'}\vp(\bm{\theta},\vx_{i,c}) - \sum_{c'=1}^Cp_{i,c,c'}\ve_{c'}}\\
    & = \frac{1}{n_cC}\mathcal{J}_{i,c}\paren{\vp(\bm{\theta},\vx_{i,c}) - \vp(\bm{\theta},\vx_{i,c})}\\
    & = 0
\end{align*}
Therefore, we have that
\[
    \tilde{\vg}_c = -\frac{1}{\omega_c}\sum_{i\in\mathcal{I}_c}\omega_{i,c,c}\tilde{\vg}_{i,c,c} = -\frac{1}{\omega_cn_cC}\sum_{i\in\mathcal{I}_c}p_{i,c,c}\tilde{\vg}_{i,c,c} = -\frac{1}{\omega_cn_cC}\sum_{i\in\mathcal{I}_c}\vp\paren{\bm{\theta},\vx_{i,c}}_c\hat{\vg}_i\paren{\bm{\theta}}
\]
Redefine $\vg_c' = \frac{1}{n_c}\sum_{i\in\mathcal{I}_c}\vp\paren{\bm{\theta},\vx_{i,c}}_c\hat{\vg}_i\paren{\bm{\theta}}$. Since $\vg_c$'s are linearly independent, we have that
\[
    \text{span}\paren{\vg_{\text{class}}} = \text{span}\paren{\tilde{\vg}_1,\dots,\tilde{\vg}_C} = \text{span}\paren{\vg_1',\dots,\vg_C'}
\]
\end{proof}

\section{Proof of Theorem~\ref{thm:g_subspace_diff}}\label{app:thm4_proof}
\begin{proof}
Let $\vq\in\text{span}\paren{\vg_1\paren{\bm{\theta}},\dots\vg_C\paren{\bm{\theta}}}^\top$ be a unit vector. Then we must have that $\vq \in \text{col}\paren{\bm{\Sigma}_{\text{top}}\paren{\bm{\theta}}}^\perp$. By definition, we have that
\[
    \paren{\vq^\top\vg_c'\paren{\bm{\theta}}}^2 = \paren{\vq^\top\paren{\vg_c'\paren{\bm{\theta}} - n_c\bar{p}_c\vg_c\paren{\bm{\theta}}}}^2 \leq \frac{1}{4}\vq^\top\bm{\Sigma}_c\paren{\bm{\theta}}\vq 
\]
where the last inequality follows Lemma~\ref{lem:q_inner_diff}.
Let $\vg\in\R^{d\times C}$ be the matrix such that the $c$th column is $\sqrt{\frac{n_c}{ n}}\vg_c'\paren{\bm{\theta}}$. Then we have that
\[
    \norm{\vg^\top\vq}_2^2 = \sum_{c=1}^C\frac{n_c}{ n}\paren{\vq^\top\vg_c'\paren{\bm{\theta}}}^2 \leq \frac{1}{4}\vq^\top\sum_{c=1}^C\frac{n_c}{ n}\bm{\Sigma}_c\paren{\bm{\theta}}\vq = \frac{1}{4}\vq^\top\bm{\Sigma}\paren{\bm{\theta}}\vq 
\]

Notice that $\mI - \vu\vu^\top$ is the projection operator onto $\text{span}\paren{\vg_1\paren{\bm{\theta}},\dots\vg_C\paren{\bm{\theta}}}^\perp$. Then we have that
\[
    \norm{\paren{\mI - \vu\vu^\top}\vg}_2^2 \leq \sup_{\vq\in \text{span}\paren{\vg_1\paren{\bm{\theta}},\dots\vg_C\paren{\bm{\theta}}}^\top}\norm{\vg^\top\vq}_2^2\leq \frac{1}{4} \sup_{\vq\in\text{col}\paren{\bm{\Sigma}_{\text{top}}\paren{\bm{\theta}}}^\perp}\vq^\top\bm{\Sigma}\paren{\bm{\theta}}\vq \leq \frac{1}{4}\norm{\bm{\Sigma}_{\text{bulk}}\paren{\bm{\theta}}}_2
\]
Therefore
\[
    \norm{\paren{\mI - \vu\vu^\top}\vu'}_2^2 \leq \frac{1}{\sigma_{\min}\paren{\vg}^2}\norm{\paren{\mI - \vu\vu^\top}\vg}_2^2 = \frac{\norm{\bm{\Sigma}_{\text{bulk}}\paren{\bm{\theta}}}_2}{4\sigma_{\min}\paren{\vg}^2}
\]
This implies that
\[
    \sigma_{\min}\paren{\vu^\top\vu'}^2 \geq 1 - \frac{\norm{\bm{\Sigma}_{\text{bulk}}\paren{\bm{\theta}}}_2}{4\sigma_{\min}\paren{\vg}^2}
\]
\end{proof}

\section{Auxiliary Lemmas}\label{app:aux_lemmas}
\begin{lemma}
    \label{lem:matrix_cauchy}
    Let $\mM\in\R^{d\times d}$ be a symmetric positive semi-definite matrix. Then for all matrices $\vq\in\R^{d\times m_1}$ and $\mP\in\R^{d\times m_2}$ we have that
    \[
        \norm{\vq^\top\mM\mP}_2^2\leq \norm{\vq^\top\mM\vq}_2\norm{\mP^\top\mM\mP}_2
    \]
\end{lemma}
\begin{proof}
    For any vector $\mA,\vb\in\R^d$ such that $\mA^\top\mM\vb \geq 0$ we have that
    \[
        \paren{\mA^\top\mM\vb}^2 = \inner{\mM^{\frac{1}{2}}\mA}{\mM^{\frac{1}{2}}\vb}^2 \leq \norm{\mM^{\frac{1}{2}}\mA}_2^2\norm{\mM^{\frac{1}{2}}\vb}_2^2 = \mA^\top\mM\mA\cdot\vb^\top\mM\vb
    \]
    Now, let $\mA = \vq\vu$ and $\vb = \mP\vv$ with $\norm{\vu}_2 = \norm{\vv}_2 = 1$. Then we have that
    \[
        \paren{\vu^\top\vq^\top\mM\mP\vv}^2 \leq \paren{\vu^\top\vq^\top\mM\vq\vu}\paren{\vv^\top\mP^\top\mM\mP\vv} \leq \norm{\vq^\top\mM\vq}_2\norm{\mP^\top\mM\mP}_2
    \]
    Since the inequality above holds for all $\vu,\vv$ with unit norm, we can conclude that
    \[
        \norm{\vq^\top\mM\mP}_2^2\leq \norm{\vq^\top\mM\vq}_2\norm{\mP^\top\mM\mP}_2
    \]
\end{proof}

\begin{lemma}\label{lem:q_inner_diff}
    Let $\vg_c = \frac{1}{n_c}\sum_{i\in\mathcal{I}_c}\hat{\vg}_i$ and let $\vg_c' = \frac{1}{n_c}\sum_{i\in\mathcal{I}_c}\vp\paren{\bm{\theta},\vx_{i}}_c\hat{\vg}_i$ for a fixed $\bm{\theta}$. Then for any vector $\vq$ we have that
    \[
        \paren{\vq\top\paren{\vg_c' -\bar{p}_c\vg_c\paren{\bm{\theta}}}}^2\leq  \frac{1}{4}\vq^\top\bm{\Sigma}_c\paren{\bm{\theta}}\vq
    \]
    where $\bar{p}_c = \frac{1}{n_c}\sum_{i\in\mathcal{I}_c}\vp\paren{\bm{\theta},\vx_i}_c$ and $\bm{\Sigma}_c\paren{\bm{\theta}}$ is defined in (\ref{eq:cov_c_def}).
\end{lemma}
\begin{proof}
    Recall the definition
    \[
        \vg_c\paren{\bm{\theta}} = \frac{1}{n_c}\sum_{i\in\mathcal{I}_c}\hat{\vg}_i\paren{\bm{\theta}};\quad \vg_c'\paren{\bm{\theta}} = \frac{1}{n_c}\sum_{i\in\mathcal{I}_c}\vp\paren{\bm{\theta},\vx_i}_c\hat{\vg}_i\paren{\bm{\theta}}
    \]
    Fix $\bm{\theta}$. Let $\bar{p}_c := \frac{1}{n_c}\sum_{i\in\mathcal{I}_c}\mP\paren{\bm{\theta},\vx_i}_c$. Then we have that
    \begin{align*}
        \vg_c'\paren{\bm{\theta}} & = \frac{1}{n_c}\sum_{i\in\mathcal{I}_c}\vp\paren{\bm{\theta},\vx_i}_c\paren{\hat{\vg}_i\paren{\bm{\theta}} - \vg_c\paren{\bm{\theta}}} + \bar{p}_c\vg_c\paren{\bm{\theta}}\\
        & = \frac{1}{n_c}\sum_{i\in\mathcal{I}_c}\paren{\vp\paren{\bm{\theta},\vx_i}_c -\bar{p}_c}\paren{\hat{\vg}_i\paren{\bm{\theta}} - \vg_c\paren{\bm{\theta}}} + \bar{p}_c\vg_c\paren{\bm{\theta}}
    \end{align*}
    For any vector $\vq\in\R^p$ we have that
    \begin{align*}
        & \paren{\vq\top\paren{\vg_c' - \bar{p}_c\vg_c\paren{\bm{\theta}}}}^2\\
        & \quad\quad = \paren{\frac{1}{n_c}\sum_{i\in\mathcal{I}_c}\paren{\vp\paren{\bm{\theta},\vx_i}_c - \bar{p}_c}\vq^\top\paren{\hat{\vg}_i\paren{\bm{\theta}} - \vg_c\paren{\bm{\theta}}}}^2\\
        & \quad\quad \leq \paren{\frac{1}{n_c}\sum_{i\in\mathcal{I}_c}\paren{\vp\paren{\bm{\theta},\vx_i}_c - \bar{p}_c}^2}\vq^\top\paren{\frac{1}{n_c}\sum_{i\in\mathcal{I}_c}\paren{\hat{\vg}_i\paren{\bm{\theta}} - \vg_c\paren{\bm{\theta}}}\paren{\hat{\vg}_i\paren{\bm{\theta}} - \vg_c\paren{\bm{\theta}}}^\top}\vq
    \end{align*}
    Notice that 
    \[\frac{1}{n_c}\sum_{i\in\mathcal{I}_c}\paren{\vp\paren{\bm{\theta},\vx_i}_c - \bar{p}_c}^2 = \text{Var}\paren{\vp\paren{\bm{\theta},\vx_i}_c} \leq \bar{p}_c(1-\bar{p}_c) \leq \frac{1}{4}
    \]
    In the meantime
    \[
        \frac{1}{n_c}\sum_{i\in\mathcal{I}_c}\paren{\hat{\vg}_i\paren{\bm{\theta}} - \vg_c\paren{\bm{\theta}}}\paren{\hat{\vg}_i\paren{\bm{\theta}} - \vg_c\paren{\bm{\theta}}}^\top = \bm{\Sigma}_c\paren{\bm{\theta}} - \paren{\vg_c\paren{\bm{\theta}} - \vg\paren{\bm{\theta}}}\paren{\vg_c\paren{\bm{\theta}} - \vg\paren{\bm{\theta}}}^\top
    \]
    Therefore
    \begin{align*}
        \paren{\vq\top\paren{\vg_c' -\bar{p}_c\vg_c\paren{\bm{\theta}}}}^2\leq \frac{1}{4}\paren{\vq^\top\bm{\Sigma}_c\paren{\bm{\theta}}\vq - \paren{\vq^\top\paren{\vg_c\paren{\bm{\theta}} - \vg\paren{\bm{\theta}}}}^2}\leq \frac{1}{4}\vq^\top\bm{\Sigma}_c\paren{\bm{\theta}}\vq
    \end{align*}
\end{proof}

\end{document}